\documentclass{article}

\usepackage{microtype}
\usepackage{graphicx}
\usepackage{subfigure}
\usepackage{booktabs} 
\usepackage{nameref}
\usepackage{amsfonts}
\usepackage{amssymb}
\usepackage{amsmath}
\usepackage{amsthm}
\usepackage{mathtools}
\newtheorem{theorem}{Theorem}[section]

\newtheorem{lemma}[theorem]{Lemma}
\newtheorem{proposition}{Proposition}[section]
\DeclareMathOperator*{\argmax}{arg\,max}
\DeclareMathOperator*{\argmin}{arg\,min}
\usepackage{algorithm}
\usepackage{tikz}
\usepackage{algorithmic}
\usepackage{framed}
\usepackage{balance}
\definecolor{light-blue}{rgb}{.27, .37, .6}
\definecolor{light-brown}{rgb}{.8, .4, .3}
\usepackage{tabularx}
\newcolumntype{M}[1]{>{\centering\arraybackslash}m{#1}}
\newcolumntype{N}{@{}m{0pt}@{}}
\usepackage{float}

\usepackage{hyperref}

\usepackage[accepted]{icml2019}

\icmltitlerunning{Generative Adversarial Imitation from Observation}

\begin{document}

\twocolumn[
\icmltitle{Generative Adversarial Imitation from Observation}



\icmlsetsymbol{equal}{*}

\begin{icmlauthorlist}
	\icmlauthor{Faraz Torabi}{ut}
	\icmlauthor{Garrett Warnell}{army}
	\icmlauthor{Peter Stone}{ut}
\end{icmlauthorlist}

\icmlaffiliation{ut}{University of Texas at Austin, Austin, USA}
\icmlaffiliation{army}{U.S. Army Research Laboratory}

\icmlcorrespondingauthor{Faraz Torabi}{faraztrb@cs.utexas.edu}

\icmlkeywords{Machine Learning, ICML}

\vskip 0.3in
]



\printAffiliationsAndNotice{\icmlEqualContribution} 

\begin{abstract}
{\em Imitation from observation} (\mbox{\emph{IfO}}) is the problem of learning directly from state-only demonstrations without having access to the demonstrator's actions.
The lack of action information both distinguishes \mbox{\emph{IfO}} from most of the literature in imitation learning, and also sets it apart as a method that may enable agents to learn from a large set of previously inapplicable resources such as internet videos. 
In this paper, we propose both a general framework for \mbox{\emph{IfO}} approaches and also a new \mbox{\emph{IfO}} approach based on generative adversarial networks called \textit{generative adversarial imitation from observation} (\mbox{\emph{GAIfO}}). We conduct experiments in two different settings: \emph{(1)} when demonstrations consist of low-dimensional, manually-defined state features, and \emph{(2)} when demonstrations consist of high-dimensional, raw visual data. We demonstrate that our approach performs comparably to classical imitation learning approaches (which have access to the demonstrator's actions) and significantly outperforms existing imitation from observation methods in high-dimensional simulation environments.
\end{abstract}

\section{Introduction}
One well-known way in which artificially-intelligent agents are able to learn to perform tasks is via {\em reinforcement learning (\mbox{RL})} \cite{sutton1998reinforcement} techniques. Using these techniques, if agents are able to interact with the world and receive feedback (known as {\em reward}) based on how well they are performing with respect to a particular task, they are able to use their own experience to improve their future behavior. However, designing a proper feedback mechanism for complex tasks can sometimes prove to be extremely difficult for system designers. Moreover, learning based solely on one's own experience can be exceedingly slow.

Concerns such as the ones above have given rise to the study of {\em imitation learning} \cite{schaal1997learning,billard2008robot,argall2009survey}, where agents instead attempt to learn a task by observing another, more expert agent perform that task. Because the information about how to perform the task is communicated to the imitating agent via a demonstration, this paradigm does not require the explicit design of a reward function. Moreover, because the demonstrations directly provide rich information regarding how to perform the task correctly, imitation learning is typically faster than \mbox{\emph{RL}}. While there are multiple ways that this problem can be formulated, one general approach is referred to as \textit{inverse reinforcement learning (\mbox{IRL})} \cite{russell1998learning}. \mbox{\emph{IRL}}-based techniques aim to first infer the expert agent's reward function, and then learn imitating behavior using \mbox{\emph{RL}} techniques that utilize the inferred function.

Importantly, most of the imitation learning literature has thus far concentrated only on situations in which the imitator not only has the ability to observe the demonstrating agent's {\em states} (e.g., observable quantities such as spatial location), but also the ability to observe the demonstrator's {\em actions} (e.g., internal control signals such as motor commands). While this extra information can make the imitation learning problem easier, requiring it is also limiting. In particular, requiring action observations makes a large number of valuable learning resources -- e.g., vast quantities of online videos of people performing different tasks \cite{zhou2017procnets} -- useless. For the demonstrations present in such resources, the actions of the expert are unknown. This limitation has recently motivated work in the area of \textit{imitation from observation (\mbox{IfO})} \cite{liu2017imitation}, in which agents seek to perform imitation learning using state-only demonstrations. 

Broadly speaking, the \mbox{\emph{IfO}} problem consists of two major subproblems: (1) perception of the demonstrations, i.e., extracting useful features from raw visual data, and (2) learning a control policy using the extracted features. Most \mbox{\emph{IfO}} work thus far \cite{liu2017imitation,sermanet2017time} has focused on perception and not on control. While powerful methods for perceiving the demonstrations have been developed, the control problem is solved via relatively simple means, i.e., reinforcement learning over a pre-defined reward function. Depending on the defined reward function, this approach could be restrictive, as discussed further in the next section. Therefore, we seek a more sophisticated control algorithm that is able to learn the task automatically from the demonstrations without explicitly defining a reward function.

In this paper, we propose a general framework for the control aspect of \mbox{\emph{IfO}} in which we characterize the cost as a function of state transitions only. Under this framework, the \mbox{\emph{IfO}} problem becomes one of trying to recover the state-transition cost function of the expert. Inspired by the work of \citeauthor{ho2016generative} (\citeyear{ho2016generative}), we introduce a novel, model-free algorithm called {\em generative adversarial imitation from observation} (\mbox{\emph{GAIfO}}) and prove that it is a specific version of the general framework proposed for \mbox{\emph{IfO}}. We then experimentally evaluate \mbox{\emph{GAIfO}} in high-dimensional simulation environments in two different settings: (1) demonstrations and states of the imitator are manually-defined features, and (2) demonstrations and states of the imitator come exclusively from raw visual observation. We show that the proposed method compares favorably to other recently-developed methods for \mbox{\emph{IfO}} and also that it performs comparably to state-of-the-art conventional imitation learning methods that {\em do} have access the the demonstrator's actions.

The rest of this paper is organized as follows. In Section \ref{sec:related-work}, we cover related work in imitation learning and review existing research in imitation from observation. Then, we present the notation and background needed in Section \ref{preliminaries}. In Section \ref{sec:general}, we introduce our proposed general framework for \mbox{\emph{IfO}} problems and, in Sections \ref{sec:gaifo} and \ref{sec:gaifo:impl}, we discuss our \mbox{\emph{IfO}} algorithm, \mbox{\emph{GAIfO}}. Finally, we describe and discuss our experiments in Sections \ref{sec:experiments} and \ref{sec:results}, respectively.

\section{Related Work}\label{sec:related-work}
Because our work is related to imitation learning \cite{schaal2003computational}, we first discuss here different approaches and recent advancements in this area. In general, existing work in imitation learning can be split into two categories: {\em (1)} behavioral cloning (\mbox{\emph{BC}}) \cite{bain1995a,pomerleau1989alvinn}, and {\em (2)} inverse reinforcement learning (\mbox{\emph{IRL}}) \cite{ng2000algorithms,abbeel2004apprenticeship,ziebart2008maximum,fu2017learning}.

Behavioral cloning methods use supervised learning as a means by which to find a direct mapping from states to actions. \mbox{\emph{BC}} approaches have been used to successfully learn many different tasks such as navigation for quadrotors \cite{giusti2016machine} or autonomous ground vehicles \cite{bojarski2016end}. Inverse reinforcement learning (\mbox{\emph{IRL}}) techniques, on the other hand, seek to learn the demonstrator's cost function and then use this learned cost function in order to learn an imitation policy through \mbox{\emph{RL}} techniques. \mbox{\emph{IRL}} methods have been used for interesting tasks such as dish placement and pouring \cite{finn2016guided}. To the best of our knowledge, the current state of the art in imitation learning is an \mbox{\emph{IRL}}-based technique called generative adversarial imitation learning (\mbox{\emph{GAIL}}) \cite{ho2016generative}. \mbox{\emph{GAIL}} uses generative adversarial networks (\mbox{\emph{GANs}}) \cite{goodfellow2014generative} as a means by which to bring the distribution of state and action pairs of the imitator and the demonstrator closer together. 

Most existing imitation learning approaches require demonstrations that include the expert actions. However, these actions are not always observable, and sometimes it is more practical to be able to imitate state-only demonstrations. One step towards this goal is the work of \citeauthor{finn2017one} (\citeyear{finn2017one}) where a meta-learning imitation learning method is proposed that enables a robot to reuse past experience and learn new skills from a single demonstration. In particular, raw pixel videos are used as the source of demonstration information. However, it is still assumed that the expert actions are available during meta-training; the requirement for actions is only lifted at test time when learning the new task.

One way to approach the aforementioned problem is to ``learn to imitate" (as opposed to imitation learning), i.e., by doing some preprocessing, enable the agent to follow a single demonstration exactly. Two such approaches are proposed by \citeauthor{nair2017combining} (\citeyear{nair2017combining}) and \citeauthor{pathak2018zero} (\citeyear{pathak2018zero}). These methods first learn an inverse dynamics model through self-supervised exploration, and then use it to infer the demonstrator's action at each step and perform that in the environment. These approaches mimic the one demonstration that they are exposed to exactly (as opposed to learning and generalizing a task from multiple different demonstrations).

A second approach to imitation from action-free demonstrations is \textit{behavioral cloning from observation (\mbox{BCO})} \cite{torabi2018behavioral}. This method also learns an inverse dynamics model through self-supervised exploration which is then used to infer actions from demonstrations. The problem is then treated as a regular imitation learning problem, and behavioral cloning is used to learn an imitation policy that maps states to the inferred actions. Therefore, this method is able to learn and generalize from different demonstrations but, since it is based on behavioral cloning, it may suffer from the well-studied compounding error caused by covariate shift \cite{ross2010efficient,ross2011reduction,laskey2016shiv}.

A third class of techniques that is able to perform imitation learning without requiring knowledge of actions includes those that first focus on learning a representation of the task and then use an \mbox{\emph{RL}} method with a predefined surrogate reward over that representation. For example, \citeauthor{gupta2017learning} (\citeyear{gupta2017learning}) have proposed an invariant feature space to transfer skills between agents with different embodiments, \citeauthor{liu2017imitation} (\citeyear{liu2017imitation}) have presented a network architecture which is capable of handling differences in viewpoints and contexts between the imitator and the demonstrator, and \citeauthor{sermanet2017time} (\citeyear{sermanet2017time}) have proposed a time-contrastive network which is invariant to both different embodiments and viewpoints. While these techniques represent significant advances in representation learning, each of them uses the same surrogate reward function, i.e., the proximity of the imitator's and demonstrator's encoded representation at each time step. One of the downsides of this reward function is that each provided demonstration needs to be time-aligned, i.e., at every time step, each demonstration needs to have advanced to the same point of the task. 
Another approaches developed by \citeauthor{merel2017learning} and \citeauthor{henderson2018optiongan} aim to imitate the state distribution of the expert. However, the state distribution does not represent the demonstrator policy and the learned policy may fail in tasks such as the cyclic ones. Moreover, these approaches have thus far focused mostly on experimentation and less on the theoretical underpinnings of the control problem.
In our work, we propose a new algorithm to remove the constraints mentioned above, and also provide theoretical analysis of this approach.

\section{Preliminaries}\label{preliminaries}

\paragraph{Notation} We consider agents within the framework of Markov decision processes (MDPs). In this framework, $\mathcal{S}$ and $\mathcal{A}$ are the state and action spaces, respectively. An agent at a particular state $s \in \mathcal{S}$, chooses an action $a \in \mathcal{A}$, based on a policy $\pi:\mathcal{S}\times\mathcal{A}\rightarrow[0,1]$ and transitions to state $s'$ with probability of $P(s'|s,a)$ that is predefined by the environment transition dynamics. In this process, the agent gets feedback $c(s,a)$ which is coming from a cost function $c:\mathcal{S}\times\mathcal{A}\rightarrow\mathbb{R}$. In this paper, $\overline{\mathbb{R}}$ means the extended real numbers $\mathbb{R}\cup \{+\infty\}$ and expectation over a policy means the expectation over all the trajectories that it generates.

\paragraph{Inverse Reinforcement Learning (\mbox{\emph{IRL}})}

As described earlier,  one general approach to imitation learning is based on IRL. The first step of this approach is to learn a cost function based on the given state-action demonstrations. This cost function is learned such that it is minimal for the trajectories demonstrated by the expert and maximal for every other policy \cite{abbeel2004apprenticeship}. However, since the problem is underconstrained --- many policies can lead to the same (demonstrated) trajectories --- another constraint is usually assigned as well which chooses the policy that has the maximum entropy. This method is called maximum entropy inverse reinforcement learning (\mbox{\emph{MaxEnt IRL}}) \cite{ziebart2008maximum}. A very general form of this framework can be described as
\begin{equation}\label{irl}
\begin{split}
\mbox{\emph{IRL}}_\psi(\pi_E) = &\displaystyle{\argmax_{c \in \mathbb{R}^{\mathcal{S} \times \mathcal{A}}}} -\psi(c) +  (\displaystyle{\min_{\pi \in \prod}}-\lambda_H H({\pi}) + \\ &\mathbb{E}_\pi[c(s,a)])-\mathbb{E}_{\pi_E}[c(s,a)]\;,
\end{split}
\end{equation}

where $\psi(c):\mathbb{R}^{\mathcal{S}\times\mathcal{A}}\rightarrow\overline{\mathbb{R}}$ is a convex cost function regularizer, $\pi_E$ is the expert policy, $\mathrm{\Pi}$ is the space of all the possible policies, and $H(\pi)$ and $\lambda_H$ are the entropy function of the policy $\pi$ and its weighting parameter respectively. The output here is the desired cost function. The second step of this framework is to input the learned cost function into a standard reinforcement learning problem. An entropy-regularized version of \mbox{\emph{RL}} can be described as
\begin{align}
\mbox{\emph{RL}}(c) = \displaystyle{\argmin_{\pi \in \prod}} - \lambda_H H(\pi) + \mathbb{E}_\pi[c(s,a)]\;,
\end{align}
which aims to find a policy that minimizes the cost function and maximizes the entropy.

\paragraph{Generative Adversarial Imitation Learning (\mbox{\emph{GAIL}})} Recently, \citeauthor{ho2016generative} have shown that by considering a specific function as the cost regularizer $\psi(c)$, the described pipeline ((1) and (2)) can be solved instead as
\begin{equation}\label{gail}
\begin{split}
\min_{\pi \in \prod} \displaystyle{\max_{D \in (0,1)^{\mathcal{S} \times \mathcal{A}}}} & -\lambda_H H(\pi) + \mathbb{E}_\pi[\log(D(s,a)] +\\ &\mathbb{E}_{\pi_E}[\log(1-D(s,a))]\;,
\end{split}
\end{equation}
where $D:\mathcal{S} \times \mathcal{A} \rightarrow (0,1)$ is a classifier trained to discriminate between the state-action pairs that arise from the demonstrator and the imitator. Excluding the entropy term, the loss function in (\ref{gail}) is similar to the loss of generative adversarial networks \cite{goodfellow2014generative}. Instead of first learning the cost function and then learning the policy on top of that, this method directly learns the optimal policy by bringing the distribution of the state-action pairs of the imitator as close as possible to that of the demonstrator.

\section{A General Framework for Imitation from Observation}\label{sec:general}

In \mbox{\emph{IRL}}, both states and actions are available and the goal is to find a cost function that on average has a smaller value for the trajectories generated by the expert policy compared to the ones generated by any other policy. In the case of imitation from observation, however, the demonstrations that the agent receives are limited to the expert's state-only trajectories. In the context of the \mbox{\emph{IRL}}-based approaches to imitation learning discussed above, this lack of action information makes it impossible to calculate the $\mathbb{E}_{\pi_E}[c(s,a)]$ term in (\ref{irl}). Consequently, none of the approaches described in Section \ref{preliminaries} is directly applicable in this setting.

In imitation from observation, the goal is for the imitator to perform similarly to the expert in the environment, i.e., for the actions of the demonstrator and imitator to have the same effect on the environment (performing the task), rather than taking exactly the same actions. Therefore, instead of characterizing the cost signal as a function of states and actions $c: \mathcal{S} \times \mathcal{A} \rightarrow \mathbb{R}$, we define them as a function of the state transitions $c: \mathcal{S} \times \mathcal{S} \rightarrow \mathbb{R}$. Based on this characterization, we formulate inverse reinforcement learning from observation as
\begin{equation}\label{irlfo}
\begin{split}
\mbox{\emph{IRLfO}}_\psi(\pi_E) = &\displaystyle{\argmax_{c \in \mathbb{R}^{\mathcal{S} \times \mathcal{S}}}} -\psi(c) + (\displaystyle{\min_{\pi \in \prod}} \ \mathbb{E}_\pi[c(s,s')])\\
&-\mathbb{E}_{\pi_E}[c(s,s')])\;,
\end{split}
\end{equation}
which outputs $\tilde{c}$. Note that in (\ref{irlfo}) we ignore the entropy term so as to simplify the theoretical analysis presented in Section \ref{sec:gaifo}. Evidence form \citeauthor{ho2016generative} suggests that doing so is fine from an empirical perspective (they set $\lambda_H = 0$ in more than $80\%$ of their successful experiments). We leave detailed analysis of the effect of this term to future work. From a high-level perspective, in imitation from observation, the goal is to enable the agent to extract what the task is by observing some state sequences. Intuitively, this extraction is possible because we expect the beneficial state transitions for any given task to form a low-dimensional manifold within the $\mathcal{S} \times \mathcal{S}$ space.  Thus, the intuition behind our definition of the cost function is to penalize based on how close each transition is to that manifold.

Now using an \mbox{\emph{RL}} algorithm for $\tilde{c}$ amounts to solving:
\begin{align}\label{modified-rl}
 \mbox{\emph{RL}}(\tilde{c}) = \displaystyle{\argmin_{\pi \in \prod}} \ \mathbb{E}_\pi[\tilde{c}(s,s')]
\end{align}
where the output, $\tilde{\pi}$, is the imitation policy. 

\section{Generative Adversarial Imitation from Observation}\label{sec:gaifo}

Having developed the general framework in (\ref{irlfo}), we now propose a specific algorithm, generative adversarial imitation from observation (\mbox{\emph{GAIfO}}). To this end, we first define the state-transition occupancy measure, $\rho_\pi^{s}:\mathcal{S} \times \mathcal{S} \rightarrow \mathbb{R}$ as
\begin{equation}\label{occupancy-s}
\rho^{s}_\pi(s_i,s_j) = \sum_a P(s_j|s_i,a) \pi(a|s_i) \sum_{t=0}^{\infty} \gamma^t P(s_t=s_i|\pi)\;.
\end{equation} 
This occupancy measure corresponds to the distribution of state transitions that an agent encounters when using policy $\pi$. We define the set of valid state-transition occupancy measures as $\mathrm{P}^{s} \triangleq \{\rho^{s}_\pi:\pi \in \mathrm{\Pi}\}$.

We now introduce a proposition which is the foundation of our algorithm. In the following proposition we use the convex conjugate concept which is defined as follows: for a function $f:X \rightarrow \overline{\mathbb{R}}$, the convex conjugate $f^*:X^* \rightarrow \bar{\mathbb{R}}$ is defined as $f^*(x^*) \triangleq \sup_{x \in X} \langle x^*,x \rangle -f(x)$.

\begin{proposition}\label{main} 
$\mbox{RL} \circ \mbox{IRLfO}_\psi (\pi_E)$ and $\argmin_{\pi \in \mathrm{\Pi}} \psi^*(\rho_\pi^{s} - \rho^{s}_{\pi_E})$ induce policies that have the same state-transition occupancy measure, $\rho^s_{\tilde{\pi}}$.
\end{proposition}

In the rest of this section, we prove this proposition and then by choosing a specific regularizer, we present our algorithm. At the end we propose a practical implementation of the algorithm. To prove the proposition, we first define another problem, $\overline{\mbox{\emph{RL}}} \circ \overline{\mbox{\emph{IRLfO}}}_\psi (\pi_E)$, and argue that it outputs a state-transition occupancy measure which is the same as $\rho^s_{\tilde{\pi}}$ induced by $\mbox{\emph{RL}} \circ \mbox{\emph{IRLfO}}_\psi (\pi_E)$. We define
\begin{align}\label{alternative-irlfo}
\begin{split}
\overline{\mbox{\emph{IRLfO}}}_\psi(\pi_E) = &\displaystyle{\argmax_{c \in \mathbb{R}^{\mathcal{S} \times \mathcal{S}}}}  (\displaystyle{\min_{\rho_\pi^{s} \in \mathrm{P}^{s}}} \sum_{s,s'} \rho_\pi^{s} (s,s') c(s,s')) \\
&-\sum_{s,s'}  \rho^{s}_{\pi_E} (s,s') c(s,s')-\psi(c)\;,
\end{split}
\end{align}
where, the output is a cost function $\bar{c}$. Note that, $\mathbb{E}_\pi[c(s,s')] = \sum_{s,s'}\rho_\pi^{s} (s,s') c(s,s')$ so (\ref{irlfo}) and (\ref{alternative-irlfo}) are similar except that the former is optimized over $\pi \in \mathrm{\Pi}$ and the latter over $\rho_\pi^{s} \in \mathrm{P}^{s}$. If we consider using an \mbox{\emph{RL}} method to find a state-transition occupancy measure under $\tilde{c}$, (\ref{modified-rl}) can be rewritten as
\begin{align}\label{rl*}
\overline{\mbox{\emph{RL}}}(\bar{c}) = \displaystyle{\min_{\rho_\pi^{s} \in \mathrm{P}^{s}}} \sum_{s,s'} \rho^{s}_\pi (s,s') \bar{c}(s,s')\;,
\end{align}
which would now output the desired state-transition occupancy measure $\bar{\rho}^s_\pi$.
 
\begin{lemma}\label{first}
$\overline{\mbox{RL}} \circ \overline{\mbox{IRLfO}}_\psi (\pi_E)$ outputs a state-transition occupancy measure, $\bar{\rho}_\pi^s$, which is the same as $\rho^s_{\tilde{\pi}}$ induced by $\mbox{RL} \circ \mbox{IRLfO}_\psi (\pi_E)$.
\end{lemma}
\begin{proof}
From the definition of $\mathrm{P}^s$, the mapping from $\mathrm{\Pi}$ to $\mathrm{P}^{s}$ is surjective, i.e., for every $\rho_\pi^{s} \in \mathrm{P}^{s}$, there exists at least one $\pi \in \mathrm{\Pi}$. Therefore, we can say $\bar{\rho}_\pi^s = \rho^s_{\tilde{\pi}}$ (where $\tilde{\pi}$ and $\bar{\rho}_\pi^s$, as already defined, are the outputs of (\ref{modified-rl}) and (\ref{rl*}), and $\rho^s_{\tilde{\pi}}$ is the state-transition occupancy measure that corresponds to $\tilde{\pi}$). Therefore, solving $\mbox{\emph{RL}} \circ \mbox{\emph{IRLfO}}_\psi (\pi_E)$ results in the same $\rho^s_\pi$ as applying $\overline{\mbox{\emph{RL}}}$ using the cost function returned by $\overline{\mbox{\emph{IRLfO}}}$ in (\ref{alternative-irlfo}).
\end{proof}
Note that, in this lemma, the returned policies from these two problems are not necessarily the same. The reason is that the mapping from $\mathrm{\Pi}$ to $\mathrm{P}^{s}$ is not injective, i.e., there could be one or multiple $\pi \in \mathrm{\Pi}$ that corresponds to the same $\rho_\pi^{s} \in \mathrm{P}^{s}$. Consequently, it is not necessarily the case that a policy that gives rise to $\bar{\rho}_\pi$ is the same as $\tilde{\pi}$. However, as we discussed in the previous section, in imitation from observation, we are primarily concerned with the effect of the policy on the environment so this situation is acceptable.

Now we introduce another lemma that helps us in the proof of Proposition \ref{main}.

\begin{lemma}\label{dual}
$\overline{\mbox{RL}} \circ \overline{\mbox{IRLfO}}_\psi (\pi_E) = \argmin_{\rho_\pi^{s} \in \mathrm{P}^{s}} \psi^*(\rho_\pi^{s} - \rho^{s}_{\pi_E})$
\end{lemma}

This lemma is proven in the appendix \footnote{The appendix is anonymously presented https://tinyurl.com/ybkn8v7n
https://tinyurl.com/ybkn8v7n
} using the minimax principle \cite{millar1983minimax}. Thus far, by combining Lemmas \ref{first} and \ref{dual}, we can conclude that $\rho^s_{\tilde{\pi}}$ induced by $\mbox{RL} \circ \mbox{IRLfO}_\psi (\pi_E)$ is the same as the output of $\argmin_{\rho_\pi^{s} \in \mathrm{P}^{s}} \psi^*(\rho_\pi^{s} - \rho^{s}_{\pi_E})$. Now, we only need one more step to prove Proposition \ref{main}:

\begin{lemma}\label{second}
$\argmin_{\pi \in \mathrm{\Pi}} \psi^*(\rho_\pi^{s} - \rho^{s}_{\pi_E})$ is a policy that has a state-transition occupancy measure that is the same as the output of
$\argmin_{\rho_\pi^{s} \in \mathrm{P}^{s}} \psi^*(\rho_\pi^{s} - \rho^{s}_{\pi_E})$.
\end{lemma}

The proof of Lemma \ref{second} is similar to that of Lemma \ref{first}. Now based on Lemmas \ref{first}, \ref{dual}, and \ref{second} we can conclude that Proposition \ref{main} holds.

Having proved this proposition, we can solve $\argmin_{\pi \in \mathrm{\Pi}} \psi^*(\rho_\pi^{s} - \rho^{s}_{\pi_E})$ instead of $\mbox{RL} \circ \mbox{IRLfO}_\psi (\pi_E)$. To this end, we consider the generative adversarial regularizer
\begin{align}\label{psi}
\psi_{GA}(c) \triangleq
\begin{cases}
  \mathbb{E}_{\pi_E} [g(c(s,s'))] & \text{if } c<0\\    
  +\infty & \text{otherwise}
\end{cases} & 
\end{align}
where
\begin{align}
g(x) =
\begin{cases}
  -x-\log(1-e^x) & \text{if } x<0\\    
  +\infty & \text{otherwise}
\end{cases} \;
\end{align}
which is a closed, proper, convex function and has convex conjugate
\begin{align}
\label{psi*}
\begin{split}
\psi_{GA}^*(\rho_\pi^{s} - \rho^{s}_{\pi_E}) = &\displaystyle\max_{D \in (0,1)^{\mathcal{S} \times\mathcal{S}}} \sum_{s,s} \rho_\pi^{s}(s,s')\log(D(s,s'))+ \\
&\rho^{s}_{\pi_E}(s,s')\log(1-D(s,s')) \; ,
\end{split}
\end{align}
where $D:\mathcal{S} \times \mathcal{S} \rightarrow (0,1)$ is a discriminative classifier.
A similar convex conjugate is derived in \citeauthor{ho2016generative}; However, for the sake of completeness, we prove the properties claimed for (\ref{psi}) and show that (\ref{psi*}) is its convex conjugate in the appendix. \footnote{This proof closely follows the proofs of Proposition A.1. and Corollary A.1.1. of \citeauthor{ho2016generative} and it is included here for the sake of completeness. The only substantive difference is that in our case we consider state-transition occupancy measure $(s,s')$ instead of $(s,a)$.}

\begin{algorithm}[t!]
	\caption{\mbox{\emph{GAIfO}}}\label{alg:GAIfO}
	\begin{algorithmic}[1]
		\STATE Initialize parametric policy $\pi_\phi$ with random $\phi$ 
		\STATE Initialize parametric discriminator $D_\theta$ with random $\theta$
		\STATE Obtain state-only expert demonstration trajectories $\tau_E=\{(s,s')\}$
		\WHILE {Policy Improves}
			\STATE Execute $\pi_\phi$ and store the resulting state transitions $\tau=\{(s,s')\}$
			\STATE Update $D_\theta$ using loss $$-\Big(\mathbb{E}_\tau [\log(D_\theta(s,s'))]+\mathbb{E}_{\tau_E} [\log(1-D_\theta(s,s'))]\Big)$$
			\STATE Update $\pi_\phi$ by performing \mbox{\emph{TRPO}} updates with reward function $$-\Big(\mathbb{E}_{\tau} [\log(D_\theta(s,s'))]\Big)$$
		\ENDWHILE
	\end{algorithmic}
\end{algorithm}

\tikzset{%
  state neuron/.style={
    circle,
    draw,
    fill=blue,
    minimum size=.1cm
  },
  next state neuron/.style={
    circle,
    draw,
    fill={rgb:blue,1;white,1},
    minimum size=.1cm
  },
  policy neuron/.style={
    circle,
    draw,
    fill=green,
    minimum size=.1cm
  },
  discriminator neuron/.style={
    circle,
    draw,
    fill=yellow,
    minimum size=.1cm
  },
  action neuron/.style={
    circle,
    draw,
    fill=red,
    minimum size=.1cm
  },
  value neuron/.style={
    circle,
    draw,
    fill=brown,
    minimum size=.1cm
  },
  neuron missing/.style={
    fill=none,
    draw=none, 
    scale=1,
    execute at begin node=\color{black}$\hdots$
  },
}

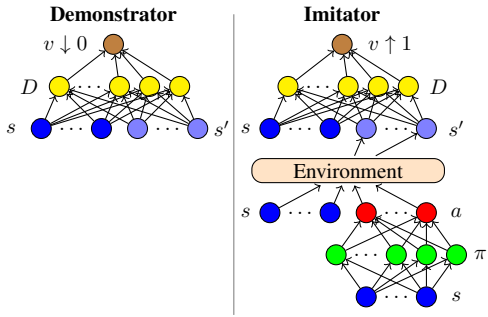
\begin{figure}
\centering
\begin{tikzpicture}[scale=0.8, every node/.style={transform shape}]
\draw[gray] (2.9,-1) -- (2.9, 4);

\node [align=center] at (.9,4) {\textbf{Demonstrator}};
\node [align=center] at (4.7,4) {\textbf{Imitator}};

\draw [fill={rgb:yellow,1;red,1;white,7}, rounded corners](3.2,1.2) rectangle (6.4,1.6) node [pos=0.5, align=center] (environment) {Environment};

\foreach \m/\l [count=\x] in {1,missing,2}
  \node [state neuron/.try, neuron \m/.try ] (state1-\m) at (6.6-\x*.5,-.7) {};
\node [align=center] at (6.6,-.7) {$s$};

\foreach \m [count=\x] in {1,2,3,missing,4}
  \node [policy neuron/.try, neuron \m/.try ] (policy-\m) at (7.1-\x*0.5,0) {};
\node [align=center] at (7,0) {$\pi$};

\foreach \m [count=\x] in {1,missing,2}
  \node [action neuron/.try, neuron \m/.try ] (action1-\m) at (6.6-\x*0.5,.7) {};
\node [align=center] at (6.6,.7) {$a$};
  
\foreach \m/\l [count=\x] in {1,missing,2}
  \node [state neuron/.try, neuron \m/.try] (state2-\m) at (5-\x*.5,.7) {};
\node [align=center] at (3.1,.7) {$s$};
  
\foreach \m/\l [count=\x] in {1,missing,2}
  \node [state neuron/.try, neuron \m/.try] (state3-\m) at (5-\x*.5,2.1) {};
\node [align=center] at (3.1,2.1) {$s$};
  
\foreach \m/\l [count=\x] in {1,missing,2}
  \node [next state neuron/.try, neuron \m/.try] (next_state3-\m) at (6.6-\x*.5,2.1) {};
\node [align=center] at (6.6,2.1) {$s'$};

\foreach \m [count=\x] in {1,2,3,missing,4}
  \node [discriminator neuron/.try, neuron \m/.try ] (discriminator3-\m) at (6.3-\x*0.5,2.8) {};
\node [align=center] at (6.3,2.8) {$D$};

\foreach \m [count=\x] in {1}
  \node [value neuron/.try, neuron \m/.try ] (value3-\m) at (5.2-\x*0.5,3.5) {};
  \node [align=center] at (5.5,3.5) {$v\uparrow 1$};
  
\foreach \m/\l [count=\x] in {1,missing,2}
  \node [state neuron/.try, neuron \m/.try] (state4-\m) at (1.2-\x*.5,2.1) {};
\node [align=center] at (-.8,2.1) {$s$};
  
\foreach \m/\l [count=\x] in {1,missing,2}
  \node [next state neuron/.try, neuron \m/.try] (next_state4-\m) at (2.8-\x*.5,2.1) {};
\node [align=center] at (2.7,2.1) {$s'$};

\foreach \m [count=\x] in {1,2,3,missing,4}
  \node [discriminator neuron/.try, neuron \m/.try ] (discriminator4-\m) at (2.5-\x*0.5,2.8) {};
\node [align=center] at (-.5,2.8) {$D$};

\foreach \m [count=\x] in {1}
  \node [value neuron/.try, neuron \m/.try ] (value4-\m) at (1.4-\x*0.5,3.5) {};  
  \node [align=center] at (0.1,3.5) {$v\downarrow 0$};

\foreach \i in {1,...,2}
  \foreach \j in {1,...,4}
    \draw [->] (state1-\i) -- (policy-\j);
\foreach \i in {1,...,4}
  \foreach \j in {1,...,2}
    \draw [->] (policy-\i) -- (action1-\j);
\foreach \i in {1,...,2}
  \foreach \j in {1,...,4}
    \draw [->] (state3-\i) -- (discriminator3-\j);
\foreach \i in {1,...,2}
  \foreach \j in {1,...,4}
    \draw [->] (next_state3-\i) -- (discriminator3-\j);
\foreach \i in {1,...,4}
  \foreach \j in {1}
    \draw [->] (discriminator3-\i) -- (value3-\j);
\foreach \i in {1,...,2}
  \foreach \j in {1,...,4}
    \draw [->] (state4-\i) -- (discriminator4-\j);
\foreach \i in {1,...,2}
  \foreach \j in {1,...,4}
    \draw [->] (next_state4-\i) -- (discriminator4-\j);
\foreach \i in {1,...,4}
  \foreach \j in {1}
    \draw [->] (discriminator4-\i) -- (value4-\j);
%

%
\foreach \i in {1,...,2}
  \draw [->] (action1-\i) -- (environment);
\foreach \i in {1,...,2}
  \draw [->] (state2-\i) -- (environment);
\foreach \i in {1,...,2}
  \draw [->] (environment) -- (next_state3-\i);

\end{tikzpicture}
\caption{A diagrammatic representation of \mbox{\emph{GAIfO}}. On the left, $s$ (dark blue) and $s'$ (light blue) are the state and next-state features in a demonstration transition, respectively. On the right, dark blue neurons represent the imitator's states. Based on policy $\pi$ (green), it performs action $a$ (red) in the environment, and encounters the next-state $s'$ (light blue). We aim to find a policy that generates state-transitions close to the demonstrations. To this end, we iteratively train the discriminator and the policy. The discriminator is trained in a way to output values $v$ (brown) close to zero for the data coming from the expert (left) and close to one for the data coming from the imitator (right). The policy is trained to generate state-transitions close to the demonstrations so that the discriminator is not able to distinguish them from the demonstrations.}
\label{fig:framework}
\end{figure}

Using the above, the imitation from observation problem can be solved as:
\begin{equation}\label{eqn:gaifoprogram}
\begin{split}
\displaystyle{\min_{\pi \in \mathrm{\Pi}}} \ \psi_{GA}^*(\rho^{s}_\pi - \rho^{s}_{\pi_E}) = &\displaystyle{\min_{\pi \in \mathrm{\Pi}}}\displaystyle\max_{D \in (0,1)^{\mathcal{S} \times\mathcal{S}}} \mathbb{E}_\pi [\log(D(s,s'))]+ \\
& \mathbb{E}_{\pi_E} [\log(1-D(s,s'))]
\end{split}
\end{equation}
We can see that the loss function in (\ref{eqn:gaifoprogram}) is similar to the generative adversarial loss.
We can connect this to general \mbox{\emph{GAN}}s if we interpret the expert's demonstrations as the real data, and the data coming from the imitator as the generated data.
The discriminator seeks to distinguish the source of the data, and the imitator policy (i.e., the generator) seeks to fool the discriminator to make it look like the state transitions it generates are coming from the expert.
The entire process can be interpreted as bringing the distribution of the imitator's state transitions closer to that of the expert. We call this process Generative Adversarial Imitation from Observation (\mbox{\emph{GAIfO}}).

\section{Practical Implementation}
\label{sec:gaifo:impl}
Based on the preceding analysis, we now specify our practical implementation of the \mbox{\emph{GAIfO}} algorithm.
We represent the discriminator, $D$, using a multi-layer perceptron with parameters $\theta$ that takes as input a state transition and outputs a value between $0$ and $1$.
We represent the policy, $\pi$, using a multi-layer perceptron with parameters $\phi$ that takes as input a state and outputs an action.
We begin by randomly initializing each of these networks, after which the imitator selects an action according to $\pi_\phi$ and executes that action.
This action leads to a new state, and we feed both this state transition and the entire set of expert state transitions to the discriminator.
The discriminator is updated using the Adam optimization algorithm \cite{kingma2014adam}, with cross-entropy loss that seeks to push the output for expert state transitions closer to $0$ and the imitator's state transitions closer to $1$.
After the discriminator update, we perform trust region policy optimization (\mbox{\emph{TRPO}}) \cite{schulman2015trust} to improve the policy using a reward function that encourages state transitions that yield small outputs from the discriminator (i.e., those that appear to be from the demonstrator).
This process continues until convergence.
The algorithm is shown in Algorithm \ref{alg:GAIfO} and the framework is summarized in Figure \ref{fig:framework}.

The implementation described above is only effective for cases in which the demonstration consists of low-dimensional state representations.
In particular, the imitation policy maps a single state to the imitating action and the reward function operates on a single state transition.
This approach is feasible for cases in which \emph{(a)} the states can be assumed to be fully-observable, and \emph{(b)} the system is strictly Markovian.
However, when considering visual state representations, neither of these assumptions is necessarily valid.
Therefore, agents operating in such state spaces are typically provided instead a recent state history.
This is useful because, for example, having knowledge about the velocity of the agent at each time step is important in order to select the correct action, and velocity information is not available when considering a single image.
Therefore, we propose here a second implementation of \mbox{\emph{GAIfO}} that enables imitation from visual demonstration data.
It modifies the implementation used for low-dimensional state representations by adding convolutional layers and using images from multiple time steps as the input to the generator and discriminator. This implementation is summarized in Figure \ref{fig:C-GAIfO}.

\begin{figure*}[!ht]
\centering
\begin{tikzpicture}[thick,scale=0.65, every node/.style={scale=0.65}]
	\draw[use as bounding box, transparent] (-1.8,-6.9) rectangle (13, 4.3);
	\draw[opacity=.8,draw=black, rounded corners, dashdotted] (-1.8,-.6) -- (-1.8,4.3) -- (13.5,4.3) -- (13.5,-.6) -- (-1.8,-.6);
	\node [align=center] at (-.6, 3.8) {\huge Policy};
	\draw[opacity=.8,draw=black, rounded corners, dashdotted] (-1.8,-6.9) -- (-1.8,-.7) -- (13.5,-.7) -- (13.5,-6.9) -- (-1.8,-6.9);
	\node [align=center] at (0.2, -1.2) {\huge Discriminator};
	\draw[fill=light-blue,opacity=.8,draw=black] (-1.3,.4) -- (-1.3,2.4) -- (.7,2.4) -- (.7,.4) -- (-1.3,.4);
	\node[] (input image) at (0,1) {\includegraphics[height=20mm]{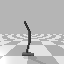}};
	\draw[draw=black] (-1,0) -- (-1,2) -- (1,2) -- (1,0) -- (-1,0);
	\draw[<->, line width=0.2mm] (-1,2) -- (-1.25,2.35);
	\draw[densely dashed] (1,2) -- (.7,2.4);
	\draw[densely dashed] (-1.3,.4) -- (-1,0);
	\node [align=center] at (-.55, 2.2) {\small 4 states};
	
	\draw[<->, line width=0.2mm] (-1,-.1) -- (1,-.1);
	\node [align=center] at (0, -.3) {\small 64};
	\draw[<->, line width=0.2mm] (1.1,0) -- (1.1,2);
	\node [align=left] at (1.35, .2) {\small 64};

	\draw[->, line width=0.2mm] (1.18,.9) -- (2.18,.9);
	\node [align=left] at (1.7, 1.7) {\small 8*8 conv};
	\node [align=left] at (1.6, 1.4) {\small stride 4};
	\node [align=left] at (1.53, 1.1) {\small ReLU};

	\draw[fill=light-blue,opacity=.8,draw=black] (2.25,.8) -- (2.25,2.3) -- (3.75,2.3) -- (3.75,.8) -- (2.25,.8);
	\draw[fill=light-blue,opacity=.8,draw=black] (2.8,0) -- (2.8,1.5) -- (4.3,1.5) -- (4.3,0) -- (2.8,0);
	\draw[<->, line width=0.2mm] (2.8,1.5) -- (2.30,2.25);
	\draw[densely dashed] (3.75,2.3) -- (4.3,1.5);
	\draw[densely dashed] (3.75,.8) -- (4.3,0);
	\draw[densely dashed] (2.25,.8) -- (2.8,0);
	\node [align=center] at (3.05, 1.9) {\small 8 filters};
	
	\draw[<->, line width=0.2mm] (2.8,-.1) -- (4.3,-.1);
	\node [align=center] at (3.55, -.3) {\small 15};
	\draw[<->, line width=0.2mm] (4.4,0) -- (4.4,1.5);
	\node [align=left] at (4.65, .2) {\small 15};

	\draw[->, line width=0.2mm] (4.5,.9) -- (5.5,.9);
	\node [align=left] at (5.03, 1.7) {\small 4*4 conv};
	\node [align=left] at (4.93, 1.4) {\small stride 2};
	\node [align=left] at (4.86, 1.1) {\small ReLU};

	\draw[fill=light-blue,opacity=.8,draw=black] (5.55,1.2) -- (5.55,2.2) -- (6.55,2.2) -- (6.55,1.2) -- (5.55,1.2);
	\draw[fill=light-blue,opacity=.8,draw=black] (6.25,0) -- (6.25,1) -- (7.25,1) -- (7.25,0) -- (6.25,0);
	\draw[<->, line width=0.2mm] (6.25,1) -- (5.55,2.2);
	\draw[densely dashed] (6.25,0) -- (5.55,1.2);
	\draw[densely dashed] (6.55,2.2) -- (7.25,1);
	\draw[densely dashed] (6.55,1.2) -- (7.25,0);
	\node [align=center] at (6.4,1.6) {\small 16 filters};
	
	\draw[<->, line width=0.2mm] (6.25,-.1) -- (7.25,-.1);
	\node [align=center] at (6.75, -.3) {\small 6};
	\draw[<->, line width=0.2mm] (7.35,0) -- (7.35,1);
	\node [align=left] at (7.55, .2) {\small 6};

	\draw[->, line width=0.2mm] (7.45,.9) -- (8.25,.9);
	\node [align=left] at (7.8, 1.1) {\small ReLU};

	\draw[fill=light-blue,opacity=.8,draw=black] (8.35,-.1) -- (8.35,2) -- (8.75,2) -- (8.75,-.1) -- (8.35,-.1);
	\node [align=center] at (8.55, -.3) {\small 128};

	\draw[->, line width=0.2mm] (8.9,.9) -- (9.8,.9);

	\draw[fill=light-blue,opacity=.8,draw=black] (9.9,.4) -- (9.9,1.5) -- (10.3,1.5) -- (10.3,.4) -- (9.9,.4);
	\node [align=center] at (10.1, .2) {\small action};
	\draw[->, line width=0.2mm] (10.45,.9) -- (11.5,.9);

	\node[] (input image) at (10,3) {\includegraphics[height=20mm]{c_state.png}};
	\draw[draw=black] (9,2) -- (9,4) -- (11,4) -- (11,2) -- (9,2.);
	\draw[->, line width=0.2mm] (11.1,3) -- (11.5,3);
	
	\draw [fill={rgb:yellow,1;red,1;white,7}, rounded corners](11.7,.4) rectangle (12.4,3.5) node [pos=0.5, align=center,rotate=-90] (environment) {Environment};
	\draw[->, line width=0.2mm,draw=black, rounded corners] (12.55,1.9) -- (12.9,1.9) -- (12.9,-2.35) -- (11.4,-2.35);

	\draw[fill=light-blue,opacity=.8,draw=black] (9.3,-3.1) -- (9.3,-1.1) -- (11.3,-1.1) -- (11.3,-3.1) -- (9.3,-3.1);
	\node[] (input image) at (10,-2.5) {\includegraphics[height=20mm]{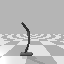}};
	\draw[draw=black] (9,-3.5) -- (9,-1.5) -- (11,-1.5) -- (11,-3.5) -- (9,-3.5);
	\draw[<->, line width=0.2mm] (9.3,-1.1) -- (9,-1.5);
	\draw[densely dashed] (11.3,-1.1) -- (11,-1.5);
	\draw[densely dashed] (11.3,-3.1) -- (11,-3.5);
	\node [align=center] at (9.8,-1.35) {\small 3 states};
	
	\draw[<->, line width=0.2mm] (9,-3.6) -- (11,-3.6);
	\node [align=center] at (10.6, -3.8) {\small 64};
	\draw[<->, line width=0.2mm] (8.9,-3.5) -- (8.9,-1.5);
	\node [align=left] at (8.7, -2.2) {\small 64};
	
	\node [align=center] at (10, -3.9) {OR};
	
	\draw[fill=light-blue,opacity=.8,draw=black] (9.3,-6.1) -- (9.3,-4.1) -- (11.3,-4.1) -- (11.3,-6.1) -- (9.3,-6.1);
	\node[] (input image) at (10,-5.5) {\includegraphics[height=20mm]{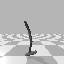}};
	\draw[draw=black] (9,-6.5) -- (9,-4.5) -- (11,-4.5) -- (11,-6.5) -- (9,-6.5);
	\draw[<->, line width=0.2mm] (9.3,-4.1) -- (9,-4.5);
	\draw[densely dashed] (11.3,-4.1) -- (11,-4.5);
	\draw[densely dashed] (11.3,-6.1) -- (11,-6.5);
	\node [align=center] at (9.8,-4.3) {\small 3 states};
	
	\draw[<->, line width=0.2mm] (9,-6.6) -- (11,-6.6);
	\node [align=center] at (10, -6.75) {\small 64};
	\draw[<->, line width=0.2mm] (8.9,-6.5) -- (8.9,-4.5);
	\node [align=left] at (8.7, -5.5) {\small 64};
	
	\node [align=center] at (12.5, -5.2) {Demonstration};
	
	\draw[->, line width=0.2mm] (8.9,-3.8) -- (7.8,-3.8);
	\node [align=left] at (8.35, -3) {\small 5*5 conv};
	\node [align=left] at (8.25, -3.3) {\small stride 2};
	\node [align=left] at (8.33, -3.6) {\small L-ReLU};

	\draw[fill=light-blue,opacity=.8,draw=black] (6.2,-4.15) -- (6.2,-2.65) -- (7.7,-2.65) -- (7.7,-4.15) -- (6.2,-4.15);
	\draw[fill=light-blue,opacity=.8,draw=black] (5.65,-4.95) -- (5.65,-3.45) -- (7.15,-3.45) -- (7.15,-4.95) -- (5.65,-4.95);
	\draw[<->, line width=0.2mm] (6.2,-2.65) -- (5.65,-3.45);
	\draw[densely dashed] (7.7,-2.65) -- (7.15,-3.45);
	\draw[densely dashed] (7.7,-4.15) -- (7.15,-4.95);
	\draw[densely dashed] (6.2,-4.15) -- (5.65,-4.95);
	\node [align=center] at (6.75, -3.05) {\small 16 filters};
	
	\draw[<->, line width=0.2mm] (5.65,-5.05) -- (7.15,-5.05);
	\node [align=center] at (6.4, -5.25) {\small 19};
	\draw[<->, line width=0.2mm] (5.55,-4.95) -- (5.55,-3.45);
	\node [align=left] at (5.35, -4.6) {\small 19};

	\draw[->, line width=0.2mm] (5.45,-3.8) -- (4.35,-3.8);
	\node [align=left] at (4.9, -3) {\small 5*5 conv};
	\node [align=left] at (4.8, -3.3) {\small stride 2};
	\node [align=left] at (4.88, -3.6) {\small L-ReLU};

	\draw[fill=light-blue,opacity=.9,draw=black] (3.25,-3.75) -- (3.25,-2.75) -- (4.25,-2.75) -- (4.25,-3.75) -- (3.25,-3.75);
	\draw[fill=light-blue,opacity=.9,draw=black] (2.55,-4.95) -- (2.55,-3.95) -- (3.55,-3.95) -- (3.55,-4.95) -- (2.55,-4.95);
	\draw[<->, line width=0.2mm] (3.25,-2.75) -- (2.55,-3.95);
	\draw[densely dashed] (4.25,-2.75) -- (3.55,-3.95);
	\draw[densely dashed] (4.25,-3.75) -- (3.55,-4.95);
	\draw[densely dashed] (3.25,-3.75) -- (2.55,-4.95);
	\node [align=center] at (3.77, -3.05) {\small 32 filters};
	
	\draw[<->, line width=0.2mm] (2.55,-5.05) -- (3.55,-5.05);
	\node [align=center] at (3.05, -5.25) {\small 7};
	\draw[<->, line width=0.2mm] (2.45,-4.95) -- (2.45,-3.95);
	\node [align=left] at (2.30, -4.45) {\small 7};

	\draw[->, line width=0.2mm] (2.35,-3.8) -- (1.25,-3.8);
	\node [align=left] at (1.8, -3) {\small 5*5 conv};
	\node [align=left] at (1.7, -3.3) {\small stride 2};
	\node [align=left] at (1.78, -3.6) {\small L-ReLU};

	\draw[fill=light-blue,opacity=.9,draw=black] (.45,-3.55) -- (.45,-2.85) -- (1.15,-2.85) -- (1.15,-3.55) -- (.45,-3.55);
	\draw[fill=light-blue,opacity=.9,draw=black] (-.05,-4.85) -- (-.05,-4.15) -- (.65,-4.15) -- (.65,-4.85) -- (-.05,-4.85);
	\draw[<->, line width=0.2mm] (.45,-2.85) -- (-.05,-4.15);
	\draw[densely dashed] (1.15,-2.85) -- (.65,-4.15);
	\draw[densely dashed] (1.15,-3.55) -- (.65,-4.85);
	\draw[densely dashed] (.45,-3.55) -- (-.05,-4.85);
	\node [align=center] at (.8, -3.05) {\scriptsize 64 filters};
	
	\draw[<->, line width=0.2mm] (-.05,-4.95) -- (.65,-4.95);
	\node [align=center] at (.3, -5.15) {\small 1};
	\draw[<->, line width=0.2mm] (-.15,-4.85) -- (-.15,-4.15);
	\node [align=left] at (-.3, -4.5) {\small 1};

	\draw[->, line width=0.2mm] (-.25,-3.8) -- (-1,-3.8);
	
	\draw[fill=light-blue,opacity=.9,draw=black] (-1.5,-4) -- (-1.5,-3.6) -- (-1.1,-3.6) -- (-1.1,-4) -- (-1.5,-4);
	\node [align=left] at (-1.7, -3.8) {v};
	
\end{tikzpicture}
\caption{A diagrammatic representation of our \mbox{\emph{GAIfO}} implementation for processing visual state representations. A stack of $4$ grayscale images from $t-3$ to $t$ ($t$ being the current time-step) enters the policy \mbox{\emph{CNN}} (top left). The policy outputs an action that the agent takes in the environment and goes to the next state in time $t+1$ (top right). A stack of $3$ grayscale images from $t-1$ to $t+1$ of the agent is prepared along with a stack of $3$ consecutive state images (grayscale) of the demonstrator (bottom right). When data from the imitation policy is provided, the stack from the imitator enters the discriminator and outputs the reward for taking that action (bottom left). This reward value is then used to both update the policy using \mbox{\emph{TRPO}} and also update the discriminator using supervised learning (to drive the value closer to zero). When data from the demonstrator is provided, the stack from the demonstrator enters the discriminator and outputs a value which is then used to update the discriminator (to drive the value closer to one).}
\label{fig:C-GAIfO}
\end{figure*}

\begin{figure*}[ht]
	\centering
	\includegraphics[scale=.40]{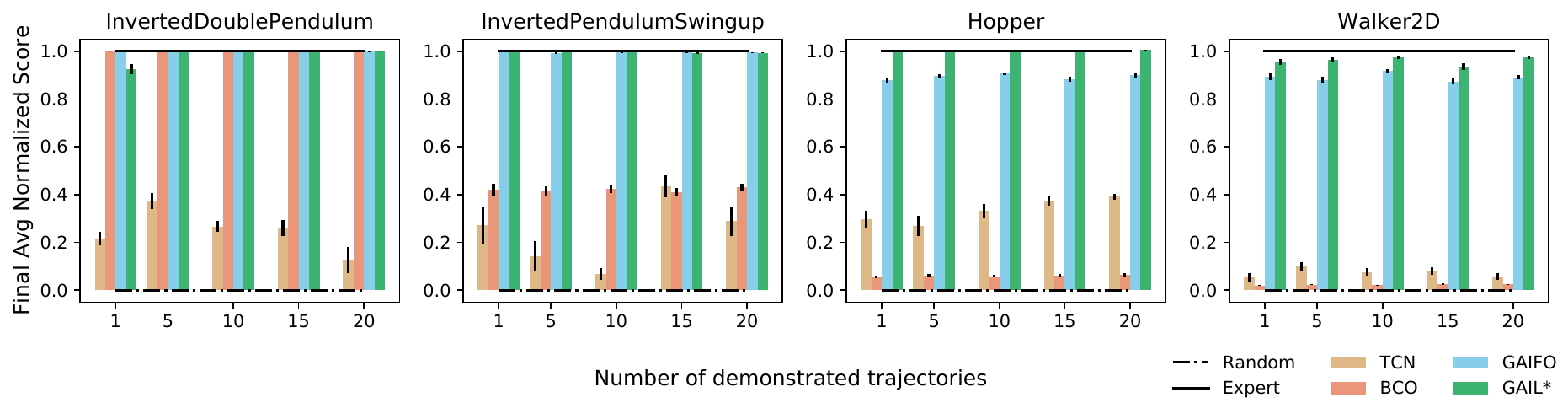}
	\caption{Performance of algorithms in low-dimensional experiments with respect to the number of demonstration trajectories. Rectangular bars and error bars represent mean return and standard deviations, respectively. For comparison purposes, we have scaled all the performances such that a random and the expert policy score $0.0$ and $1.0$, respectively. *\mbox{\emph{GAIL}} has access to action information.}
	\label{fig:performance}
\end{figure*}

\section{Experimental Setup and Implementation Details}\label{sec:experiments}

We evaluate our algorithm in domains from OpenAI Gym \cite{1606.01540} based on the Pybullet simulator \cite{coumans2016pybullet}. 
In each of the domains, we used trust region policy optimization (\mbox{\emph{TRPO}}) \cite{schulman2015trust} to train the expert agents, and we recorded the demonstrations using the resulting policy.

The results shown in the figures are the average over ten independent trials. We compare our algorithm against three baselines:
\begin{itemize}
\item \textbf{Behavioral Cloning from Observation (\mbox{\emph{BCO}})\cite{torabi2018behavioral}:} \mbox{\emph{BCO}} first learns an inverse dynamics model through self-supervised exploration, and then uses that model to infer the missing actions from state-only demonstrated trajectories. \mbox{\emph{BCO}} then uses the inferred actions to learn an imitation policy using conventional behavioral cloning. 

\item \textbf{Time Contrastive Networks (\mbox{\emph{TCN}})\cite{sermanet2017time}:} \mbox{\emph{TCN}}s use a triplet loss to train a neural network to learn an encoded form of the task at each time step. This loss function brings the states that occur in a small time-window closer together in the embedding space and pushes the ones from distant time-steps far apart. A reward function is then defined as the Euclidean distance between the embedded demonstration and the embedded agent's state at each time step. The imitation policy is learned using \mbox{\emph{RL}} techniques that seek to optimize this reward function. 

\item \textbf{Generative Adversarial Imitation Learning (\mbox{\emph{GAIL}}) \cite{ho2016generative}:} This method is as specified in Section \ref{preliminaries}. Note that, this method has access to the demonstrator's actions while the others do not.
\end{itemize}

\section{Results and Discussion}\label{sec:results}
In this section, we present the results of the two sets of experiments described above.

\subsection{Low-dimensional State Representations}
Figure \ref{fig:performance} illustrate the comparative performance of \mbox{\emph{GAIfO}} in our experimental domains using the low-dimensional state representations.
We can see that, for the domains considered here, \mbox{\emph{GAIfO}} {\em (a)} performs very well compared to other \mbox{IfO} techniques, and {\em (b)} is surprisingly comparable to \mbox{\emph{GAIL}} even though \mbox{\emph{GAIfO}} lacks access to explicit action information.	
	
Figure \ref{fig:performance} compares the final performance of the imitation policies learned by different algorithms.
We can clearly see that \mbox{\emph{GAIfO}} outperforms the other imitation from observation algorithms by a large margin in most of the experiments.
For the InvertedDoublePendulum domain, we can see that the \mbox{\emph{TCN}} method does not perform well at all.
We hypothesize that this is the case because TCN relies on time synchronization in order to find the imitating policy, i.e., it learns what the state should be at each time step.
However, successfully performing the InvertedDoublePendulum task requires the agent to simply keep the pendulum upright, and requiring it to time synchronize with the demonstrator may be too restrictive a requirement.
\mbox{\emph{BCO}}, on the other hand, performs very well in this domain, which demonstrates that, here, the inverse dynamics model learned by \mbox{\emph{BCO}} is accurate and that the compounding error problem is negligible.
We can see that \mbox{\emph{GAIfO}} also performs very well here, achieving performance similar to that of the expert, which shows that the algorithm has been able to extract the goal of the task and find a reasonable cost function from which to learn the policy.


	\begin{figure*}[ht]
		\centering
		\includegraphics[scale=.40]{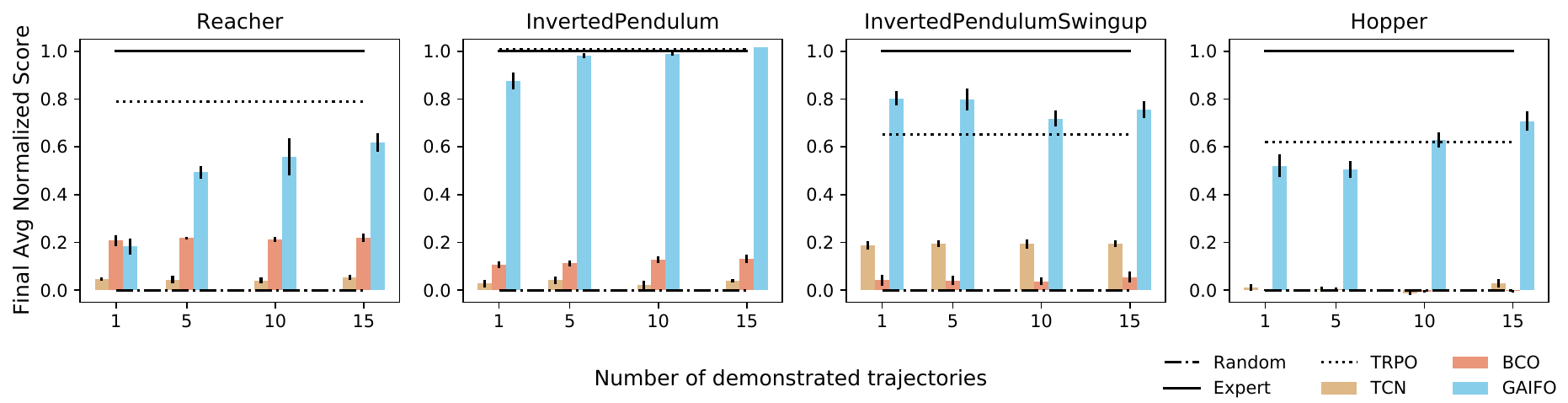}
		\caption{Performance of algorithms in visual experiments with respect to the number of demonstration trajectories. Rectangular bars and error bars represent mean return and standard deviations, respectively. For comparison purposes, we have scaled all the performances such that a random and the expert policy score $0.0$ and $1.0$, respectively.}
		\label{fig:performance-cgaifo}
	\end{figure*}

For the InvertedPendulumSwingup domain, we can see that \mbox{\emph{TCN}} again does not perform well, perhaps because the goal of the task is not well-represented in the encoding-learning phase.
\mbox{\emph{BCO}} also does not perform well.
We hypothesize that this is the case because of the compounding error problem since performing this task successfully is contingent on taking several specific actions consecutively -- deviation from those actions would cause the pendulum to drop down and not reach the goal.
\mbox{\emph{GAIfO}} and \mbox{\emph{GAIL}}, on the other hand, perform as well as the expert, which reveals that these algorithms have successfully extracted the goal and learned the task.

For both the Hopper and Walker2D domains, it can be seen that, again, \mbox{\emph{TCN}} does not work well.
We posit that this might be due to the fact that these tasks require behavior that is cyclic in nature, i.e., the expert demonstrations contain repeated states.
Because \mbox{\emph{TCN}} learns a time-dependent representation of the task, it cannot appropriately handle this periodicity and, therefore, the learned representations are not sufficient.
\mbox{\emph{GAIfO}}, however, learns a distribution of the state transitions that is not time-dependent; therefore, periodicity does not affect its performance.
\mbox{\emph{BCO}} also does not perform well in either of these two domains, perhaps again due to the compounding error problem.
Learning in these domains has two steps: first, the agent needs to learn to stand, and then the agent needs to learn to walk or hop.
With \mbox{\emph{BCO}}, it would seem that the imitating agent begins to deviate from the expert early in the task, and this early deviation ultimately leads to the imitating agent being unable to learn the secondary walking and hopping behaviors.
\mbox{\emph{GAIfO}}, on the other hand, does not suffer from this issue because it learns by executing its own policy in the environment (on-policy learning) and is therefore able to address deviation from the expert during the learning process.


\subsection{Visual State Representations}

In this section, we discuss the results of the experiments performed on the cases where the states are represented using the raw visual data.
Figure \ref{fig:performance-cgaifo} illustrates the comparison between the performance of \mbox{\emph{GAIfO}}, \mbox{\emph{BCO}} and \mbox{\emph{TCN}}.
\footnote{Here, we do not compare against \mbox{\emph{GAIL}} because doing so would require a drastic change to the structure of its discriminator in order to process raw visual data, i.e., the discriminator would need to be altered to appropriately mix action and visual data.}
In these experiments, like the ones done using the lower-dimensional state representations, the expert is trained with \mbox{\emph{TRPO}} using low-level state features, and the quantities $0$ and $1$ represent the performance of a random agent and the expert, respectively.
The demonstrations, though, consist of visual recordings using the trained policy.
Accordingly, for a more-representative baseline, we also learn a policy with \mbox{\emph{TRPO}} using visual states only (as opposed to the low-dimensional state observations) and represent the performance of that agent using a black dotted line on the plots.
This line is important in our comparison because it shows (everything being similar to \mbox{\emph{IfO}} methods) what would have been the resulting performance if the agent had access to the reward.
Figure \ref{fig:performance-cgaifo} shows that \mbox{\emph{GAIfO}} outperforms other approaches by a large margin.

It is interesting to notice that, even though \mbox{\emph{GAIfO}} (like the other IfO techniques) does not achieve the performance of the expert agent (solid line), it \emph{does} achieve the performance of the TRPO-trained agent that used visual state representations.
This suggests that, in these cases, the drop in imitation performance is perhaps due to a fundamental limitation of learning the task from visual data (i.e., partial state observability).

Finally, it can be seen that \mbox{\emph{BCO}} does not perform well in any of the domains, perhaps due to {\em (a)} the complexity of learning dynamics models over visual states, and {\em (b)} compounding error. \mbox{\emph{TCN}} also does not work well, perhaps due to the demonstrations not being time-synchronized.

\section{Conclusion and Future Work}
In this paper, we presented a general framework for imitation from observation ($\mbox{\emph{RL}}\circ\mbox{\emph{IRLfO}}_\psi(\pi_E)$) and then proposed a specific algorithm (\mbox{\emph{GAIfO}}) for doing so. 
\mbox{\emph{GAIfO}} removes the need for several restrictive assumptions that are required for some other \mbox{\emph{IfO}} techniques, including the need for multiple demonstrations to be time-synchronized.
Moreover, the on-policy nature of \mbox{\emph{GAIfO}} allows it to avoid the compounding error problem experienced by more brittle imitation techniques.
The result is an approach that is able to find better imitation policies without the need for action information, and is also able to find imitation policies that perform very close to those found by techniques that do have access to this information.

Regarding future work, note that, in our analysis, we did not consider policy entropy terms in either the \mbox{\emph{IRLfO}} step, nor in the \mbox{\emph{RL}} step. Therefore, it would be interesting to include entropy in these equations -- as has been shown to be beneficial in some cases \cite{haarnoja2017reinforcement,haarnoja2018soft} -- and investigate its effects on the overall problem and results as has been shown to be beneficial in some cases.


\section*{Acknowledgements}
This work has taken place in the Learning Agents Research
Group (LARG) at the Artificial Intelligence Laboratory, The University
of Texas at Austin.  LARG research is supported in part by grants from
the National Science Foundation (IIS-1637736, IIS-1651089,
IIS-1724157), the Office of Naval Research (N00014-18-2243), Future of
Life Institute (RFP2-000), Army Research Lab, DARPA, Intel, Raytheon,
and Lockheed Martin.  Peter Stone serves on the Board of Directors of
Cogitai, Inc.  The terms of this arrangement have been reviewed and
approved by the University of Texas at Austin in accordance with its
policy on objectivity in research.

\bibliography{example_paper}
\bibliographystyle{icml2019}

\end{document}